\documentclass[letterpaper, 10 pt, conference]{ieeeconf}
\IEEEoverridecommandlockouts                              

\usepackage{cite}
\usepackage{graphicx}
\usepackage{textcomp}
\usepackage{xcolor}
\usepackage{amsmath}
\usepackage{amssymb}
\usepackage{amsfonts}
\usepackage{mathrsfs}
\usepackage{eucal}
\usepackage{graphicx}
\usepackage{amssymb}
\usepackage{pifont}
\usepackage{float}
\usepackage{subfigure}
\usepackage{color}
\usepackage{epsfig}
\usepackage{graphicx}
\usepackage{dsfont}
\usepackage{algorithm}
\usepackage{algpseudocode}
\newtheorem{theorem}{Theorem}

\newtheorem{lemma}{Lemma}

\def\BibTeX{{\rm B\kern-.05em{\sc i\kern-.025em b}\kern-.08em
    T\kern-.1667em\lower.7ex\hbox{E}\kern-.125emX}}
\begin{document}

\title{A Robust Tube-Based Smooth-MPC for Robot Manipulator Planning}

\author{Yu Luo,
        Mingxuan Jing,
        Tianying Ji,
        Fuchun Sun*,
        Huaping Liu
\thanks{The authors are with the Department of Computer Science
and Technology, Tsinghua University, Beijing National Research Center
for Information Science and Technology (BNRist), Beijing 100084, China, (e-mail: luoyu19@mails.tsinghua.edu.cn, jmx16@mails.tsinghua.edu.cn, jity20@mails.tsinghua.edu.cn, fcsun@tsinghua.edu.cn, hpliu@mail.tsinghua.edu.cn).}
\thanks{Corresponding author is Fuchun Sun.}}

\markboth{Journal of \LaTeX\ Class Files, August~2020}%
{Shell \MakeLowercase{\textit{et al.}}: Bare Demo of IEEEtran.cls for Journals}

\maketitle

\begin{abstract}
Model Predictive Control (MPC) has shown the great performance of target optimization and constraint satisfaction. However, the heavy computation of the Optimal Control Problem (OCP) at each triggering instant brings the serious delay from state sampling to the control signals, which limits the applications of MPC in resource-limited robot manipulator systems over complicated tasks. In this paper, we propose a novel robust tube-based smooth-MPC strategy for nonlinear robot manipulator planning systems with disturbances and constraints. Based on piecewise linearization and state prediction, our control strategy improves the smoothness and optimizes the delay of the control process. By deducing the deviation of the real system states and the nominal system states, we can predict the next real state set at the current instant. And by using this state set as the initial condition, we can solve the next OCP ahead and store the optimal controls based on the nominal system states, which eliminates the delay. Furthermore, we linearize the nonlinear system with a given upper bound of error, reducing the complexity of the OCP and improving the response speed. Based on the theoretical framework of tube MPC, we prove that the control strategy is recursively feasible and closed-loop stable with the constraints and disturbances. Numerical simulations have verified the efficacy of the designed approach compared with the conventional MPC.
\end{abstract}

\section{Introduction}
Model Predictive Control (MPC), known as Receding Horizon Control (RHC), can optimize system criteria and deal with state and control input constraints effectively in industrial control process. MPC has inspired many research efforts and showed potential impacts in the academic area and the process industries compared with the methods of other multi-variable control \cite{MAYNE20142967,2020Industry,2007Model}. In each control period, MPC solves an open-loop and finite-horizon dynamic optimal control problem (OCP) to generate the optimal control input sequence and applies the first element of that sequence to the plant to update the system states. This process repeats until the system states converge around an equilibrium point, which provides an effective way to approximate the optimal control \cite{2000Constrained,SunDXL18,KohlerMA20,mesbah2016stochastic}. In this time-triggered execution fashion, MPC visits and updates the states of the system periodically to realize state feedback control and iterative convergence, which has proved the excellent control performance in the process industry, power system and robot control \cite{2016Model,BrunnerBKPZNS20,HanniganSKHYC20,rubagotti2019semi,vazquez2014model}.

\begin{figure}[t]
	\centering
	\includegraphics[width=1.0\linewidth]{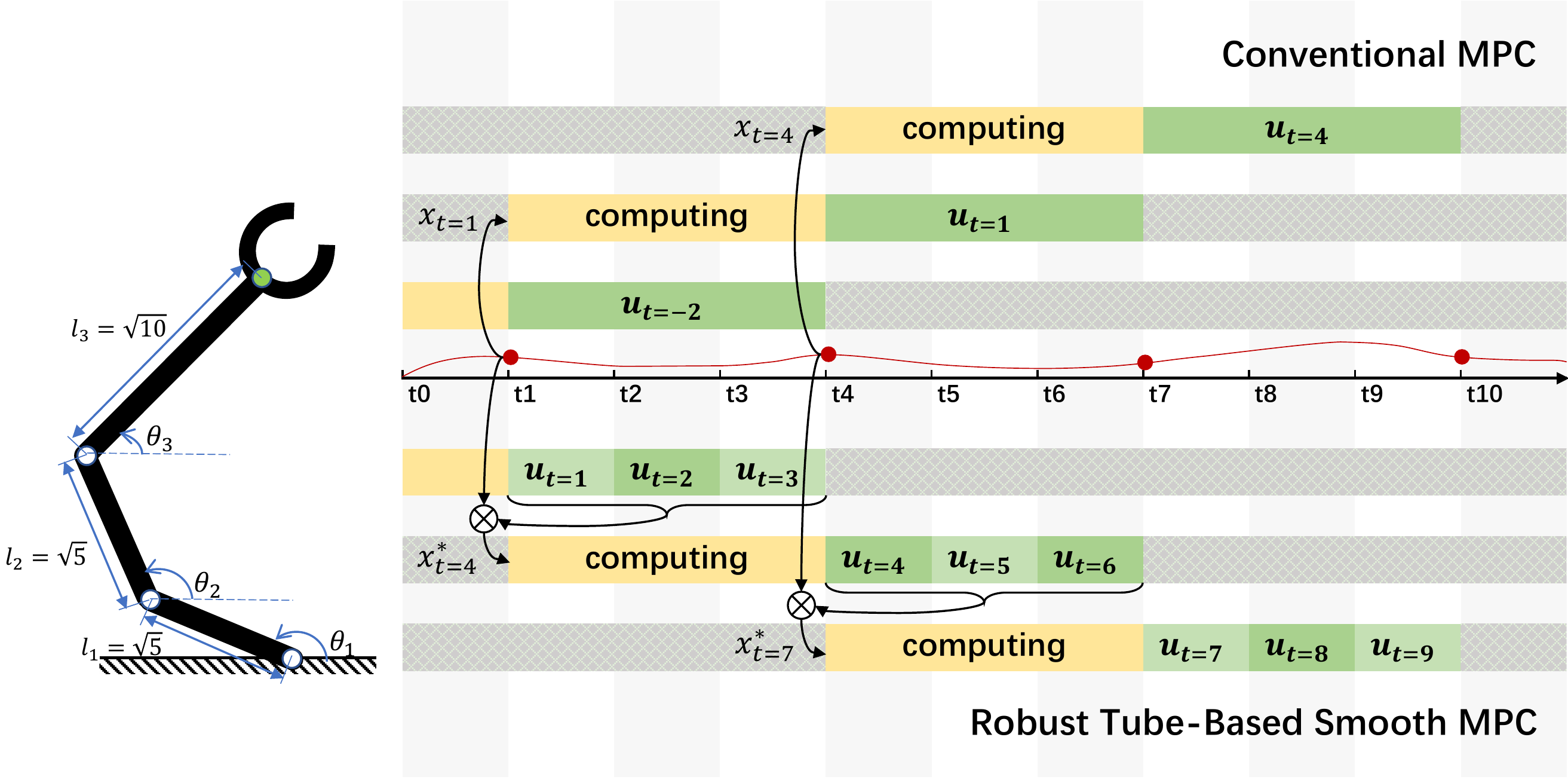}
	\caption{In this figure, we compare the implementation of the conventional MPC with the robust tube-based smooth-MPC. There are three parts to show: the three-link robot manipulator, the conventional MPC and the time-triggered MPC. For the conventional MPC, the computing time causes the multi-step delay of control input. In our control method, we predict the future state to compute the OCP ahead, which can eliminate the effect of the delay.}
	\label{fig:mcmthesis-logo}
\end{figure}

Though MPC has showed the great optimal control performance, solving the OCP periodically brings relatively heavy computation which results difficulties for implementing the controller to the real systems. Moreover, when the delay appears, the system has to hold the control input until the update of the new computational control input. In this process, the manipulator system, with the target position approaching or the reference tracking tasks, may break the constraints even though become unstable. To apply this controller into fast dynamic systems and resource-limited platforms, there are many impressive works to decrease the computation load and improve the response speed, which can be summarized as two aspects: reducing the computational complexity and decreasing the solving frequency of the OCP. In the work of Han and Tedrake \cite{Han9196824}, the method of piecewise linear affine approximations is adopted for dexterous robotic manipulation to accomplish the task with non-smooth nonlinear system and large external disturbances. Except the linearized model, shortening the prediction horizon is a effect way to reduce the computation time of the OCP in \cite{GRIFFITH2018109,lizhao9007513}. To decrease the solving frequency, event-triggered MPC \cite{hashimoto2017event} and self-triggered MPC \cite{brunner2016robust} have been one of well-developed control methods with increasing concern. In the work of Li and Liu \cite{LiS14a,liu8010327}, the framework of event-triggered MPC for continuous-time nonlinear systems is studied. By this framework, the authors in \cite{zhao2019event} develop an event-triggered decentralized tracking control with MPC for modular reconfigurable robots.

However, whatever the reduced complexity or the decreased frequency, the computation time of solving the OCP in single step is still too long to satisfy the short control period in fast systems, which causes the delay between the sampling and the input. Due to the prediction property of MPC, some references have tried to eliminate the delay by asynchronous sampling and input. In the work of \cite{zavala2009advanced}, the authors proposed the advanced-step nonlinear MPC controller to predict the next nominal system states. However, the next nominal system states deviated from the real states due to the disturbances which caused the inaccurate prediction. The authors in \cite{SU20131342} developed a dual time scale control scheme for linear/nonlinear systems with external disturbances. In this scheme, a pre-compensator and an outer MPC controller updates the control at different frequency to suppress uncertainty and ensure stability while the open-loop optimal control performance can not be guaranteed. By the characters of feed-forward action, the current control action can be calculated from the previous sampling interval for LPV model with bounded disturbance in \cite{HU202059}. In the work of \cite{lan20205235}, based on future state prediction, the MPC policy is executed in advance at current instant with the guarantee of recursive feasibility and closed-loop stability.

Motivated by these facts, we propose a novel robust tube-based smooth-MPC for nonlinear systems with constraints and disturbances, which can ensure the smoothness of the control process without potential performance degradation (Fig. 1). Different from previous works, the next real states can be predicted by multi-sampling steps, reducing the computational delay of our MPC. Further, our control strategy combines the piecewise linearization and the tube MPC, which shortens the computation time in a single control period and improves the robustness of the system. The main contributions of this work are three-fold:

(i) By deducing the deviation bound of the states in the real system and the nominal system, we give the predictive disturbed state set for the next real states as the initial condition for the next OCP at the current instant. Thus, we can start solving the next triggered OCP ahead and use the optimal results for next triggering instant to avoid the delay.

(ii) The technology of piecewise linearization in nonlinear systems is adopted to decrease the computational complexity of the OCP. Moreover, the bound of linearization error is estimated to ensure the similarity of the linear system and the nonlinear system.

(iii) Theoretical analysis on the recursive feasibility and closed-loop stability shows the effectiveness of our method.

\emph{Notation}: $\mathbb{N}$ and $\mathbb{R}$ are the natural integers and the real numbers. $\mathbb{R}^n$ means the $n$-dimension vector space. For a matrix $M$, $M\preceq0$ denotes that the real parts of all eigenvalues of $M$ are negative. For a vector $x$, $\|x\| \triangleq \sqrt{x^Tx}$ and $\|x\|_P$ with the positive definite matrix $P$ means $\|x\|_P\triangleq\sqrt{x^TPx}$. If a vector is shown as $\textbf{x}(t)$, it is a sequence $\{x(t),x(t+1),\ldots\}$. $(k+i|k)$ indicates a prediction of a variable $i$ steps ahead from time $k$. We use $\bar{\cdot}$ to mean a feasible variable satisfying all constraints and $\cdot^*$ as an optimal variable obtained by solving OCP. For any set $\mathcal{A},\mathcal{B}\in R^{n\times n}$, $\mathcal{A}\oplus\mathcal{B}$ is the Minkowski set addition, which means $\mathcal{A}\oplus\mathcal{B}\triangleq\{a+b|a\in\mathcal{A},b\in\mathcal{B}\}$ and $\mathcal{A}\ominus\mathcal{B}$ is thePontriagin set subtraction, which means $\mathcal{A}\ominus\mathcal{B}\triangleq\{a|\forall b\in\mathcal{B},a+b\in\mathcal{A}\}$.

\section{Problem Formulation}
\subsection{System Description}
In this paper, we consider a nonlinear three-link manipulator system with constraints and disturbances as follows
\begin{eqnarray}\label{sys}
\dot{z} = f(z,u)+e.
\end{eqnarray}

For this model, the manipulator has three revolute joints and three angular velocity controls in plane motion. Then, for the end operator point planning of the manipulator, the kinematics dynamic can be written as
\begin{eqnarray}\label{realsys}
\begin{bmatrix} \dot{x} \\ \dot{y} \\ \dot{\theta_1} \\ \dot{\theta_2} \\ \dot{\theta_3} \end{bmatrix}=
T(\theta)
\begin{bmatrix} \omega_1 \\ \omega_2 \\ \omega_3 \end{bmatrix} + \begin{bmatrix} e_1 \\ e_2 \\ e_3 \\ e_4\\ e_5 \end{bmatrix},
\end{eqnarray}
where $T(\theta)$ is defined as
\begin{eqnarray*}
T(\theta)=
\begin{bmatrix} -L_1 \sin(\theta_1) & -L_2 \sin(\theta_2) & -L_3 \sin(\theta_3) \\ L_1 \cos(\theta_1) & L_2 \cos(\theta_2) & L_3 \cos(\theta_3) \\
 1 & 0 & 0 \\ 0 & 1 & 0 \\ 0 & 0 & 1
\end{bmatrix}.
\end{eqnarray*}
In this model, $z(t)\triangleq(p(t),\theta(t))$ is the system states and $u(t)\triangleq(\omega_1(t),\omega_2(t),\omega_3(t))$ is the control input. $p(t)=(x(t),y(t))$ is the coordinate of the end point and $\theta(t)=(\theta_1(t),\theta_2(t),\theta_3(t))$ is the joint angle. $L_1,L_2,L_3$ are the length of the three links, respectively. $(\omega_1(t),\omega_2(t),\omega_3(t))$ is the corresponding angular velocity at each joint. For this nonlinear system,  $e(t) = (e_1(t),\ldots,e_5(t)) \in \mathbb{R}^5\cap\{0\}$ means the additional disturbances which is bounded as $\|e(t)\|\leqslant\eta_1$, where $\eta_1>0$ is a known constant. In the real control process, the system within the mechanical limitations is subjected to the following hard constraints on the control inputs and states as
\begin{eqnarray}
&z(t) \in \mathbb{Z} \subseteq \mathbb{R}^5 \triangleq \{\theta_i: \underline{\theta_i}\leqslant\|\theta_i\|\leqslant\overline{\theta_i}\}, \\
& u(t) \in \mathbb{U} \subseteq \mathbb{R}^3 \triangleq \{\omega_i: \|\omega_i\|\leqslant\overline{\omega_i}\},
\end{eqnarray}
where $\underline{\theta_i}$, $\overline{\theta_i}$ and $\overline{\omega_i}$ are three known positive constants. Beside, we introduce the nominal system of (\ref{sys}) as \begin{eqnarray}\label{nonminalsys}\label{nominalsys}
\dot{z} = f(z,u),
\end{eqnarray}
to obtain the robustness by this deterministic system and the original system.

\subsection{Control Objective}
The control objective is to move the end point from initial position to a desired final position within a reference trajectory. In order to guarantee the satisfaction of various constraints and optimal control performance, we first introduce conventional robust MPC to complete this task. Define $T$ as the prediction horizon. At each sampling instant $t_k$, the conventional MPC solves an OCP to obtain an optimal control sequence $\mathbf{u}^*(t|t_k)$, where $t\in[t_k,t_k+T]$. Thus, we consider the cost function over the prediction horizon as
\begin{eqnarray}
J(\bar{z}(t|t_k),\mathbf{\bar{u}}(t|t_k),t_k) \!\!\!\!&=&\!\!\!\! \int^{t_k+T}_{t_k} \!\!\!\!L(\bar{z}(t|t_k),\bar{u}(t|t_k))dt \nonumber \\
&&\quad \quad+ V_f(\bar{z}(t_k+T|t_k)),
\end{eqnarray}
where $L(\bar{x}(t|t_k),\bar{u}(t|t_k)) = \|\bar{z}(t|t_k)\|^2_Q + \|\bar{u}(t|t_k)\|^2_R$ is the stage cost function and $V_f(z(t_k+T|t_k))=\|\bar{z}(t_k+T|t_k)\|^2_P$ is the terminal penalty cost function. In this function, $Q$ and $P$ are positive semi-definite matrices and $R$ is a positive definite matrix. Then, the OCP 1 can be formulated as
\begin{eqnarray}
\mathbf{u}^*(t|t_k) = {\min_{\mathbf{\bar{u}}(t|t_k)\in \mathbb{U}}} J(\bar{z}(t|t_k),\mathbf{\bar{u}}(t|t_k),t_k),
\end{eqnarray}
subject to
\begin{subequations}
\begin{align}
&\bar{z}(t_k|t_k) = z(t_k),     \\
&\dot{\bar{z}}(t|t_k) = f(\bar{z}(t|t_k),\bar{u}(t|t_k)),\\
&\bar{z}(t|t_k) \in \mathbb{Z}\ominus \mathbb{Z}_e(t), \quad \bar{u}(t|t_k)\in \mathbb{U}, \\
&\bar{z}(t_k+T|t_k) \in \mathbb{Z}_\epsilon,\quad t\in[t_k,t_k+T].
\end{align}
\end{subequations}
where $\mathbb{Z}_e(t) = \{z:\|\bar{z}\|\leqslant t\eta(1+l)^t\}$ is the tightened state constraint set in Theorem 1 of Section \uppercase\expandafter{\romannumeral3} to improve the robustness of the system and $\mathbb{Z}_\epsilon = \{z:\|\bar{z}\|_P\leqslant \epsilon, \epsilon>0\}$ is the robust terminal region.\\
\textbf{Key Problem}: Due to the serious delay of single control step for solving the OCP, the controller of robot manipulator system has to keep the control input for the waiting time, which may cause the suboptimality even the instability of the system. Thus, how to ensure the one-to-one correspondence of the real system states and the optimal control input without obvious computational time delay under the framework of MPC is the key improvement of this paper.

\section{Methodology}
In this section, the prediction of the system states is developed and the robust tube-based MPC is designed with theoretical guarantee.
\subsection{Prediction of Real System States}
Firstly, we deliver two lemmas to describe the properties of the robot manipulation system.
\begin{lemma}
The system function $f(z,u)$ is locally Lipschitz continuous with respect to $x$ and $u$. By the control inputs $u_1, u_2\in \mathbb{U}$, for $\forall z_1, z_2 \in \mathbb{X}$ the system satisfies
\begin{eqnarray}
\| f(z_1,u_1)-f(z_2,u_2) \| \leqslant l_1 \|z_1-z_2\| + l_2 \|u_1-u_2\|,
\end{eqnarray}
where $l_1=\max\{L_1,L_2,L_3\}$ and $l_2=1$ are the Lipschitz constants of the nonlinear system (\ref{sys}).
\end{lemma}
By the system model, the results are obvious and the nonlinearity of the system can be compressed by the constants $l_1$ and $l_2$, which gives the inspiration for the state prediction.

\begin{lemma}
For the nonlinear system function $f(z,u)$, \\
(i) the function $f: \mathbb{R}^n \times \mathbb{R}^m \rightarrow \mathbb{R}^n$ is a twice continuously differentiable function and $f(0,0)=0$;\\
(ii) the system model can be linearized at the each instant $t_k$ and the system matrices has the formulations as follows
\begin{eqnarray}
A_{t_k} = \frac{\partial f}{\partial z}\big|_{(z(t_k),u(t_k))},\quad B_{t_k} = \frac{\partial f}{\partial u}\big|_{(z(t_k),u(t_k))}.
\end{eqnarray}
Then, the linear system is formulated as
\begin{eqnarray}\label{linearsystem}
\dot{z} = A_{t_k}z(t) + B_{t_k}u(t) +e_t.
\end{eqnarray}
(iii) considering the piece control interval $[t_k,t_{k+1}]$, the Hessian matrix $\nabla^2 f(x,u)$ is bounded as
\begin{eqnarray}
\|\nabla^2 f(x(t),u(t))\| = \begin{vmatrix} \frac{\partial^2 f}{\partial z^2} & \frac{\partial^2 f}{\partial z \partial u} \\ \\ \frac{\partial^2 f}{\partial u \partial z} & \frac{\partial^2 f}{\partial u^2} \end{vmatrix}_{(z(t),u(t))} \leqslant \eta_R,
\end{eqnarray}
where $t\in[t_k,t_{k+1}]$.
\end{lemma}
\begin{proof}
By the nominal system $f(z,u)$, we can easily prove (i). For (ii) and (iii), the linear system matrices are defined as
\begin{equation*}
A\!\!=\!\!
 \left[
 \begin{array}{c|ccc}
 \textbf{0}_{2\times2} & -\omega_1{L}_{1}\cos{\theta}_{1} & -\omega_2{L}_{2}\cos{\theta}_{2} & -\omega_3{L}_{3}\cos{\theta}_{3} \\
                & -\omega_1{L}_{1}\sin{\theta}_{1} & -\omega_2{L}_{2}\sin{\theta}_{2} & -\omega_3{L}_{3}\sin{\theta}_{3} \\ \hline
    \textbf{0}_{3\times2} & & \textbf{0}_{3\times3} & \\
\end{array}
 \right]
\end{equation*}
\begin{eqnarray*}
 B=
 \left[
 \begin{array}{ccc}
         -L_1\sin\theta_1 & -L_2\sin\theta_2 & -L_3\sin\theta_3\\
         L_1\cos\theta_1 & L_2\cos\theta_2 & L_3\cos\theta_3\\ \hline
       & \textbf{\emph{I}}_{3\times3} & \\
 \end{array}
 \right],
\end{eqnarray*}
where $\textbf{\emph{I}}_{3\times3}\in\mathbb{R}^{3\times 3}$ is the identity matrix.

Thus, substitute the states $(z(t_k),u(t_k))$ into $(A,B)$ and we can verify (ii) and (iii).  For the Hessian Matrix $H(z,u)$, the system models are smooth functions without cuspidal points in the piecewise control closed interval, which ensure the existence of the bound $\eta_R$.
\end{proof}

By this lemma and the second-order expansion of Taylor Polynomial, we consider the linearization of the nonlinear system model (\ref{sys}) at each triggering interval, i.e., $t\in[t_k,t_{k+1}]$,
\begin{eqnarray}\label{linearation}
f(z(t),u(t)) \!\!\!\!\!\!&=&\!\!\!\!\!\! f(z(t_k),u(t_k)) + A_{t_k}\big[z(t)-z(t_k)\big] \nonumber \\
&+&\!\!\!\!\!\!B_{t_k}\big[u(t)-u(t_k)\big]  +R(z(t_k),u(t_k)) \nonumber \\
&=&\!\!\!\!\!\! A_{t_k}z(t)\!\!+\!\!B_{t_k}u(t) \!\!+\!\! \Omega \!\!+\!R(z(t_k),u(t_k)).
\end{eqnarray}
In this equation, $\Omega$ is a constant vector as the system error at the point $(z(t_k),u(t_k))$,
\begin{eqnarray}\label{omega}
\Omega = f(z(t_k),u(t_k))-(A_{t_k}z(t_k)+B_{t_k}u(t_k)).
\end{eqnarray}
Especially, $R(z(t_k),u(t_k))$ is the Lagrange Remainder of the linearization error as
\begin{eqnarray}\label{lagranre}
R(z(t_k),u(t_k)) = \nabla^2f[(z(t)-z(t_k)),(u(t)-u(t_k))] \nonumber \\
\cdot f[z(t_k)+\theta(z(t)-z(t_k)),u(t_k)+\theta(u(t)-u(t_k))],
\end{eqnarray}
where $\nabla^2f$ is the Hessian matrix of the pair $(z(t_k),u(t_k))$, $\theta\in(0,1)$ is a constant and $t\in[t_k,t_{k+1}]$. Thus, by the mean value theorem, we can obtain the bound of the Hessian matrix $\nabla^2 f(z,u)$ at each piece interval, which means that the linearization error of the nonlinear system is bounded as
\begin{eqnarray}\label{linearbound}
\|\Omega+R(z(t_k),u(t_k))\| \!\!\!& \leqslant &\!\!\! \|\Omega\|+\eta_R\Big[l_1\|z(t)-z(t_k)\| \nonumber \\
&& \quad\ +l_2\|u(t)-u(t_k)\|\Big] \nonumber \\
&\triangleq& \eta_2.
\end{eqnarray}

Recalling the system (\ref{sys}), we can add this linearization error to the additional disturbances as a total disturbances. Thus, the total disturbances $w_t$ of the system model contains two parts with an upper bound $\eta$, that is
\begin{eqnarray}
\|w_t\| &=& \|e(t)\| + \|\Omega + R(z(t),u(t))\|, \nonumber\\
&\leqslant& \eta_1 + \eta_2 \nonumber\\
&\triangleq& \eta,
\end{eqnarray}
which supplies the theoretical reliability for transforming the perturbed nonlinear system as a linear system with bigger bounded disturbances. Based on the current system states and the nominal system, we can predict the region of the m-steps future states $x(t_k+m)$ by the upper bound of the total disturbances and the nominal states $x^*(t_k+m|t_k)$, which is shown in the following theorem.
\begin{theorem}
Beyond the same control input $\mathbf{\bar{u}}^*(t_k)$, the state deviation between the nominal system and the real system from $t_k$ to $t_k+m$ is bounded by
\begin{eqnarray}
\|z_e(t_k+m)\|\leqslant m\eta(1+l)^{m}, \quad m \in [0,T],
\end{eqnarray}
where $l=l_1+1$, $\eta$ is the total disturbances and $z_e(t_k+m) \triangleq z(t_k+m)-z^*(t_k+m|t_k)$ is the deviation between the nominal system and the real disturbed system.
\end{theorem}
\begin{proof}
Let $g(z,u)=f(z,u)-x$ and we easily have
\begin{eqnarray}
\|g(z_1,u)-g(z_2,u)\| \leqslant (l_1+1)\|z_1-z_2\|.
\end{eqnarray}

For the real system states with discrete-time formalization, we have as
\begin{eqnarray}
z(t_k+i+1)-z(t_k+i)\!\!\!&=&\!\!\!g(z(t_k+i),u^*(t_k+i|t_k)) \nonumber \\
&&\quad+w(t_k+i).
\end{eqnarray}
By this recursion formula, summing up from $i=0$ to $i=m-1$ yields
\begin{eqnarray}
z(t_k+m) = z(t_k) \!\!\!&+&\!\!\! \sum_{i=0}^{m-1}g(z(t_k+i),u^*(t_k+i|t_k)) \nonumber \\
&&\quad \  +w(t_k+i).
\end{eqnarray}
Analogously, we have the same formulation for the nominal system as
\begin{equation}
z^*(t_k+m|t_k) \\
= z^*(t_k|t_k) + \sum_{i=0}^{m-1}g(z^*(t_k+i|t_k),u^*(t_k+i|t_k)).
\end{equation}
Let $l=l_1+1$. Thus, the deviation of $z(t_k+m)$ and $z^*(t_k+m|t_k)$ holds that
\begin{eqnarray}
&&\|z(t_k+m)-z^*(t_k+m|t_k)\| \nonumber\\
&&=\|\sum_{i=0}^{m-1}\Big[g(z(t_k+i),u^*(t_k+i|t_k))-g(z^*(t_k+i|t_k), \nonumber\\
&&\quad \quad \quad \quad u^*(t_k+i|t_k))\Big]+\sum_{i=0}^{m-1}w(t_k+i)\| \nonumber\\
&& \leqslant m\eta + (l+1)\sum_{i=0}^{m-1}\|z(t_k+i)-z^*(t_k+i|t_k)\|.
\end{eqnarray}
Applying Gronwall-Bellman inequality, it holds that
\begin{eqnarray}
&&\|z(t_k+m)-z^*(t_k+m|t_k)\| \nonumber\\
&& \quad\quad\quad\quad \leqslant m\eta + \sum_{i=0}^{m-1}lm\eta\prod_{j=i+1}^{m-1}(1+l), \nonumber\\
&& \quad\quad\quad\quad = m\eta(1+l)^m.
\end{eqnarray}
The proof is completed.
\end{proof}
Based on this theorem, we suppose that the upper bound of computational time for solving the OCP is $m$, where $m\in[0,T]$ is a constant by repeated trials. By triangle inequality, the future real system states can be bounded as
\begin{eqnarray}\label{realcon}
\|z(t_k+m)\| \leqslant \|z^*(t_k+m|t_k)\| + m\eta(1+l)^m.
\end{eqnarray}

Referring to the definition of disturbed invariant set, we can replace (\ref{realcon}) as
\begin{eqnarray}\label{predisinvset}
\|z(t_k+m)\| \in \mathbb{Z}_{\omega}(z^*(t_k),m),
\end{eqnarray}
where $\mathbb{Z}_{\omega}(z^*(t_k),m)$ means that from the nominal states $z^*(t_k)$, we can predict the future optimal system states $z^*(t_k+m)$ and the radius $m\eta(1+l)^m$ of the disturbed invariant set so that the real states $z(t_k+m)$ can be limited in the $\mathbb{Z}_{\omega}(x^*(t_k),m)$, named predictive disturbed state set.

\subsection{Robust Tube-based Smooth-MPC}
For the conventional robust MPC, the following lemma gives a foundational theory framework to ensure the recursive feasibility and closed-loop stability.
\begin{lemma}
For the nominal linear system, MPC is of recursive feasibility and closed-loop stability if\\
(i) the OCP has a feasible solution at the initial instant $t_0$;\\
(ii) there is a local stabilizing controller $\kappa_f(z)$ in the robust terminal region to satisfy the constraint $\forall z \in \mathbb{Z}_\epsilon, \kappa_f(z) \in \mathbb{U}$ such that
\begin{eqnarray}
\dot{V_f}(\bar{z}(t))\leqslant -L(\bar{z}(t),\kappa_f(\bar{z}(t))).
\end{eqnarray}
Moreover, if the controller is chosen as $\kappa_f(z) = Kz$, we have the Lyapunov equation of the weight matrices $Q$, $R$ and $P$, that is
\begin{eqnarray}\label{riccadi}
(A+BK)^TP+P(A+BK)\leqslant -Q^*,
\end{eqnarray}
where $Q^*=Q+K^TRK$.
\end{lemma}
\begin{algorithm}[htpb]\label{def alg1}
\caption{the Robust Tube-Based Smooth-MPC}
\begin{algorithmic}[1]
\State \emph{Offline}: Initialize the parameters $m$, $l$ of system (\ref{sys}) and set the weight matrices $Q$ and $R$. By (\ref{riccadi}), compute the terminal state feed-back gain $K$ and the weight matrix $P$. Then, defining the terminal set to satisfy (\ref{terminal}). Find the optimal solution $v^*(t_0)$ for the initial state $z(t_0)$.
\State \emph{Online}:\\ for each triggering time $t_k, k=1,2,3,\ldots$
\State \quad \quad (i) Apply the first $m$ elements of the optimal control sequence $v^*(t_{k-1})$ to (\ref{tubempc}) for the real system.
\State \quad \quad (ii) Measure the current state $z(t_k)$ and compute the predictive state $z^*(m|t_k)$ by the nominal nonlinear model (\ref{nonminalsys}) and the last optimal control sequence $v^*(t_{k-1})$. Then, estimate the  predictive disturbed state set $\mathbb{Z}_{\omega}(z^*(t_k),m)$.
\State \quad \quad (iii) Based on the state $z^*(m|t_k)$, linearize the model (\ref{nonminalsys}) to obtain system matrices $A_{t_k}$ and $B_{t_k}$ and compute the local state feed-back gain $K_{t_k}$ as (\ref{riccdilocal}).
\State \quad \quad (iv) Solve the OCP 2 to obtain the optimal control sequence $v^*(t_k)$ for the next triggering instant $t_{k+1}$.
\State \quad \quad ((ii) to (iv) are synchronous with (i))
\State \quad \quad (v) Let $k = k+1$.
\end{algorithmic}
\end{algorithm}

Because the OCP 1 has to be solved at each sampling instant and the first element of the control sequence is used to control the real systems, which causes the computational delay and the suboptimality of the control input. Thus, we redefine the control period to adapt the delay as
\begin{eqnarray}
t_{k+1}-t_k = \Delta t =m\delta, \quad m\in\mathbb{N}_{\geqslant 1},\quad m \leqslant T/\Delta t,
\end{eqnarray}
Then, for each control interval $t\in[t_k,t_{k+1}]$, the optimal control $\mathbf{u}^*(t_k)$ can be used from $u^*(t_k|t_k)$ to $u^*(t_k+m|t_k)$ other than only the first element.

By Lemma 2, we can linearize the robot manipulator system with the linearization error. Thus, the system can be modeled in the interval $t\in[t_k,t_{k+1}]$ as
\begin{eqnarray}
\dot{z}(t) = A_{t_k}z(t) + B_{t_k}u(t) + w_t.
\end{eqnarray}
With the formulation of tube MPC, the controller $u(t_k)$ is designed as
\begin{eqnarray}\label{tubempc}
u(t|t_k) = v(t|t_k) + K[z(t)-z^*(t|t_k)],
\end{eqnarray}
where $v(t_k)$ is the control decision variable for the linear system and $K$ is the state feed-back gain computed by the Riccati Equation as (\ref{riccadi}) to restrain the bounded disturbances.

Based on the linearization of the nonlinear systems and the predictive disturbed state set, we redefine the optimal control problem 2 (OCP 2), which is formulated
\begin{eqnarray}
\mathbf{v}^*(t|t_k) = {\min_{\mathbf{\bar{u}}(t|t_k)\in \mathbb{U}}} J(\bar{z}(t|t_k),\mathbf{\bar{v}}(t|t_k),t_k),
\end{eqnarray}
subject to
\begin{subequations}
\begin{align}
&z(t_{k+1}|t_k) \in \mathbb{Z}_{\omega}(z^*(t_k),m),     \\
&\dot{\bar{z}}(t|t_k) = A_{t_k}\bar{z}(t|t_k)+B_{t_k}\bar{u}(t|t_k),\\
&\bar{z}(t|t_k) \in \mathbb{Z}\ominus \mathbb{Z}_e(t),\ \ \bar{u}(t|t_k)\in \mathbb{U}\ominus K\bar{z}, \\
&\bar{z}(t_k+T|t_k) \in \mathbb{Z}_\epsilon,\quad t\in[t_k,t_k+T]. \label{terminal}
\end{align}
\end{subequations}
where $J(\bar{z}(t|t_k),\mathbf{\bar{v}}(t|t_k),t_k)$, $\mathbb{Z}_e(i)$ and $\mathbb{Z}_\epsilon$ have the same definition with the OCP 1. By solving the OCP 2, we can obtain the optimal control sequence $\mathbf{v}^*(t_k)$. Then, the first $m$ elements of the sequence and the state feed-back gain as (\ref{tubempc}) are applied into the system. Repeating the process, the system states can converge to the neighbor of the equilibrium point, which is summarized in Algorithm 1..

Due to the change of the initial condition of the OCP, the multi-step usage of the optimal control sequence and the unknown disturbances, the feasibility and the stability of the designed control strategy may be lost for the repeated control process. The following theorem are developed to ensure the implementation of the smooth-MPC.
\begin{theorem}
For the discrete-time nonlinear systems (\ref{sys}) with Assumption 1 and Assumption 2, the robust tube-based smooth-MPC is of feasibility and stability if: \\
(i) the local state feed-back gain $K_{t_k}$ satisfies
\begin{eqnarray}
A_{t_k}+B_{t_k}K_{t_k} \preceq 0.
\end{eqnarray}
(ii) the upper bound $\eta$ of the total disturbances and the computational time interval $m$ hold
\begin{eqnarray}
m\eta(1+l)^m \leqslant \sum_{i=0}^{\infty}\|(A+BK)^i\|\eta.
\end{eqnarray}
\end{theorem}
\begin{proof}
For the control strategy, we mainly design three parts: (i) the multi-step usage of the optimal control sequence; (ii) the change of the initial condition for the OCP; (iii) the linearization of the nonlinear system.

Based on the theory framework of tube MPC \cite{LANGSON2004125} for linear systems with constraints and disturbance, the feasibility and the stability of the time-triggered fashion can be ensured. Thus, we can easily prove that (ii) and (iii) are still suitable for smooth-MPC. Here, we introduce the disturbance invariant set $\mathbb{Z}_{\gamma} :=\{z|\|z\|\leqslant \gamma\}$ for each piece linearation of the system, where $\gamma$ is defined as
\begin{eqnarray}
\gamma = \sum_{i=0}^{\infty}\|(A+BK)^i\|\eta,
\end{eqnarray}
Thus, if $\mathbb{Z}_{\omega}(z^*(t_k),m)\in\mathbb{Z}_{\gamma}$ is satisfied, we can ensure the implementation of the framework of tube MPC. For the multi-step usage, if
\begin{eqnarray}
m\eta(1+l)^m \leqslant \sum_{i=0}^{\infty}\|(A+BK)^i\|\eta,
\end{eqnarray}
the state constraint holds that
\begin{eqnarray}
\|z(t_k+m)\|\!\!\!\!&\in&\!\!\!\! \|z^*(t_k+m|t_k)\| \oplus \mathbb{Z}_{\omega}(m) \nonumber \\
&\in&\!\!\!\! \|z^*(t_k+m|t_k)\| \oplus \mathbb{Z}_{\gamma} \nonumber \\
&\in&\!\!\!\! \mathbb{Z} \ominus \mathbb{Z}_{e}(m) .
\end{eqnarray}
Thus, the recursive feasibility can be ensured.

On the other hand, if the state feed-back gain $K_{t_k}$ at each sampling instant satisfies
\begin{eqnarray}\label{riccdilocal}
A_{t_k}+B_{t_k}K_{t_k} \preceq 0,
\end{eqnarray}
the real system states are contained in the disturbance invariant set of the nominal states. Then, by the analysis of the conventional tube MPC, the practical stability of our approach is ensured and the proof is completed.
\end{proof}

\section{Simulation Results}
In this section, we evaluate the efficiency of our proposed control method compared with the other two MPC controller. The first one is the optimal MPC with no computation time at each triggering instant, which represents the optimal control performance. The second control strategy is the time-triggered MPC with the worst delay for each control period, which is widely used in real systems. Consider the three links robot manipulator in a $x-y$ plane as Fig. 1 and the system parameters and the initial condition are set as
\begin{equation}
\left\{
\begin{aligned}
&L_1 = L_2 = \sqrt{5}, \quad L_3 = \sqrt{10} \\
&\theta_1 = \frac{\pi}{2}+\arcsin\frac{2}{\sqrt{5}}, \theta_2 = \frac{\pi}{2}+\arcsin\frac{1}{\sqrt{5}},\\
&\theta_3 = \arcsin\frac{1}{\sqrt{10}},\\
&p = (0,4),
\end{aligned}
\right.
\end{equation}

For this system, the joint pose is defined as the angle from $x$ positive axis to finger phalanx in anti-clockwise direction. The angular velocity of finger phalanx is positive in anti-clockwise rotation direction, on the contrary negative in clockwise rotation direction. The states and inputs constraint are given as
\begin{equation}
\left\{
\begin{aligned}
&\frac{\pi}{2}\leqslant\theta_1\leqslant\pi, 0\leqslant\theta_2\leqslant\pi, 0\leqslant\theta_3\leqslant\frac{\pi}{2},  \\
&-\frac{\pi}{16}\leqslant\omega_1,\omega_2,\omega_3\leqslant\frac{\pi}{16}.
\end{aligned}
\right.
\end{equation}
At last, the upper bound of the total disturbances is set as $\eta = 0.02$ to satisfy Lemma 1 and Theorem 2.
\begin{figure}[htbp]
	\centering
	\includegraphics[width=1.0\linewidth]{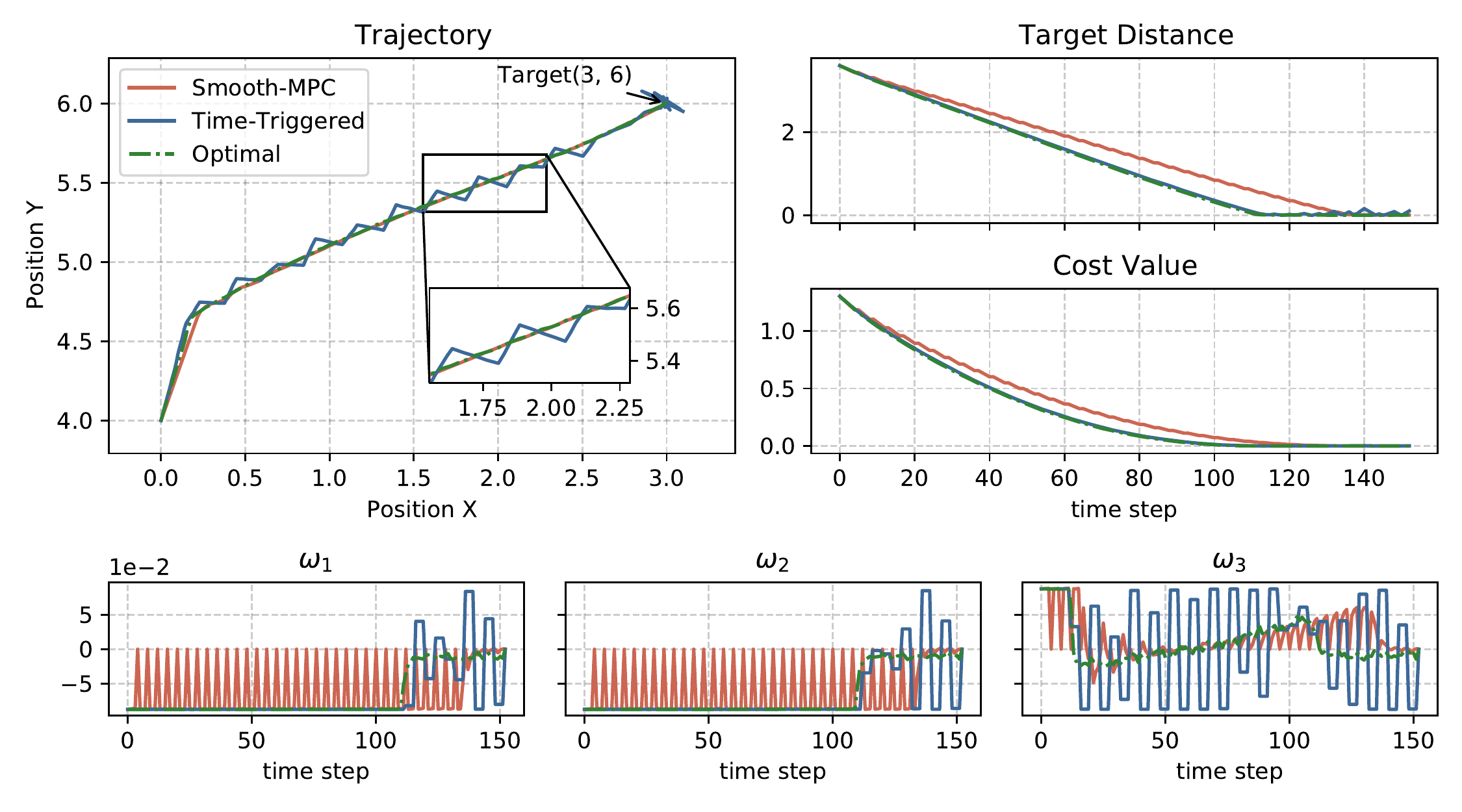}
	\caption{The comparison of three control method over the Position Tracking task, which contain the state trajectory, position error, cost function value and control input.}
\end{figure}

For the MPC controllers, the sampling time is $\delta = 0.1s$ and the prediction horizon $T=3s$. The weight matrices are defined as $Q = 0.1\textbf{\emph{I}}_{5\times5}$ and $R = 0.01\textbf{\emph{I}}_{3\times3}$, where $\textbf{\emph{I}}$ is the identity matrix. By repeated trials, the computation time is set as $m=4$, which means $t_{k+1}-t_k=0.4s$. The effectiveness of Algorithm 1 and its beneficial features are demonstrated from two tasks.
\subsection{Position Tracking}
We set the desired position of the end point at $(2,6)$. In Fig.2, the three lines represent the end point trajectory for approaching the target point. By comparison, we can find the similarity of the optimal MPC and ours, which are smoother than the time-triggered one to verify the optimal control performance of ours.
\subsection{Trajectory Tracking}
In this task, the end point tracks a specific trajectory by a quarter circle and an oblique line. From Fig. 3, our approach shows the better tracking performance and robustness than the time-triggered one.
\begin{figure}[htbp]
	\centering
	\includegraphics[width=1.0\linewidth]{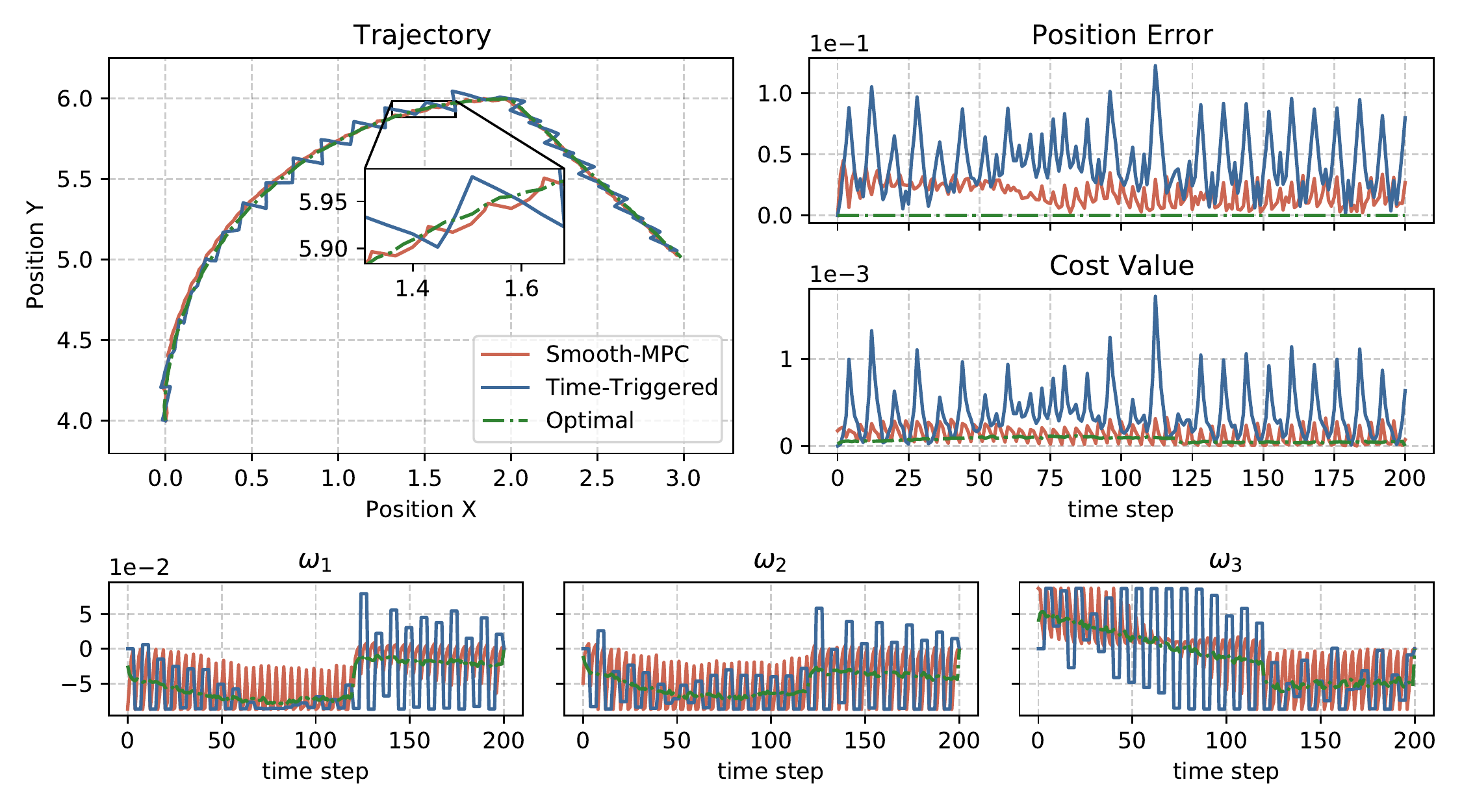}
	\caption{The comparison of three control method over the Trajectory Tracking task, which contain the state trajectory, position error, cost function value and control input.}
\end{figure}

The simulation results show that the designed MPC controller has similar excellent control performance as the optimal MPC in this two tasks using the robot manipulator system, which has less oscillation and better accuracy than time-triggered one. Besides, due to the linearization of the nonlinear system, the average computation time decreases obviously to save computational resource and improve the response speed. Thus, the effectiveness of robust tube-based smooth-MPC is verified.
\section{Conclusion}
In this paper, we aim at eliminating the delay caused by solving the OCP in the robot manipulator systems and propose a novel robust tube-based smooth-MPC to ensure the optimal control performance. We estimated the linearization error as a bounded disturbance to linearize the nonlinear system and reduce the computational complexity of the OCP. Then, the deviation of the nominal system and the real system states is deduced by Lipschitz continuity and triangle inequations to predict the region of the next real states. Based on this two mechanism, the difficulties of the delay for using MPC in fast dynamic systems are dramatically disposed and the optimality of this controller are guaranteed. The experimental results verifies the control performance and the response speed by the proposed smooth-MPC. Our future work will concentrate on the dynamic system of robot manipulator with stochastic disturbances to extend the application of this control method.


\bibliographystyle{IEEEtran}
\bibliography{reference}

\end{document}